\documentclass{article}

\usepackage{arxiv}

\usepackage{wrapfig}
\usepackage{tikz,pgfplots}
\usepackage{tcolorbox}
\usepackage{empheq}
\usepackage[utf8]{inputenc} 
\usepackage[T1]{fontenc}    
\usepackage{hyperref}       
\usepackage{url}            
\usepackage{booktabs}       
\usepackage{amsfonts}       
\usepackage{nicefrac}       
\usepackage{microtype}      
\usepackage{bbm}
\usepackage{subfig}
\usepackage{caption}
\usepackage[ruled]{algorithm2e}
\usepackage{algpseudocode}
\usepackage{graphicx}
\usepackage{microtype}

\usepackage{booktabs} 
\usepackage{amsmath,amssymb, amsthm}
\DeclareMathOperator*{\argmax}{argmax}
\DeclareMathOperator*{\argmin}{argmin}
\usepackage{bbm}
\usepackage{hyperref}
\usepackage{mathtools}
\usepackage{balance}
\usepackage{stmaryrd}

\newcommand{\indep}{\rotatebox[origin=c]{90}{$\models$}}
\newcommand{\E}{\mathbb{E}}

\newcommand{\Xset}{\mathcal{X}}
\newcommand{\realset}{\mathbb{R}}
\newcommand{\obs}{\mathbf{x}}
\newcommand{\proba}{\mathbb{P}}
\newcommand{\Dist}{\mathcal{D}}
\newcommand{\risk}{R}

\newcommand{\indicatrice}{\mathbbm{1}}

\newcommand{\auucmax}{AUUC-max}
\newcommand{\auuc}{AUUC}
\newcommand{\UM}{Uplift Modeling~}
\newcommand{\UMdl}{Uplift Model~}
\newcommand{\ITE}{ITE}

\newcommand{\graph}{\mathcal{G}}
\newcommand{\vertices}{V}
\newcommand{\edges}{E}
\newcommand{\expectation}{\mathbb{E}}

\newcommand{\cover}{{\cal C}}
\newcommand{\covers}{{\cal K}}
\newcommand{\Weight}{W}

\newcommand{\Fset}{{\cal F}}
\newcommand{\Cset}{{\cal C}}
\newcommand{\variance}{\mathbb{V}}

\newcommand{\rademacher}{\mathfrak{R}}
\newcommand{\Input}{\mathcal{X}}
\newcommand{\Output}{\mathcal{Y}}
\newcommand{\bfZ}{\mathbf{z}}

\newtheorem{theo}{Theorem}

\newtheorem {pro}{Proposition}
\newtheorem{definition}{Definition}

\title{Treatment Targeting by AUUC Maximization \\with Generalization Guarantees}

\author{
 Artem Betlei \\
  Criteo AI Lab\\
  Grenoble, France\\
  \texttt{a.betlei@criteo.com} \\
   \And
 Eustache Diemert \\
   Criteo AI Lab \\
   Grenoble, France \\
   \texttt{e.diemert@criteo.com} \\
  \And
 Massih-Reza Amini \\
   UGA/CNRS LIG \\
   Grenoble, France \\
   \texttt{Massih-Reza.Amini@imag.fr} \\
}

\begin{document}
\maketitle
\begin{abstract}
We consider the task of optimizing treatment assignment based on individual treatment effect prediction. This task is found in many applications such as personalized medicine or targeted advertising and has gained a surge of interest in recent years under the name of Uplift Modeling. It consists in targeting treatment to the individuals for whom it would be the most beneficial. In real life scenarios, when we do not have access to ground-truth individual treatment effect, the capacity of models to do so is generally measured by the Area Under the Uplift Curve (\auuc{}), a metric that differs from the learning objectives of most of the Individual Treatment Effect (\ITE) models.
  We argue that the learning of these models could inadvertently degrade \auuc{} and lead to suboptimal treatment assignment.
  To tackle this issue, we propose a generalization bound on the \auuc{} and present a novel learning algorithm that optimizes a derivable surrogate of this bound, called \auucmax{}.
  Finally, we empirically demonstrate the tightness of this generalization bound, its effectiveness for hyper-parameter tuning and show the efficiency of the proposed algorithm compared to a wide range of competitive baselines on two classical benchmarks.
\end{abstract}


\section{Introduction}

In many applications there is a need to target actions to specific portions of a population so as to maximize a global utility. For instance, in personalized medicine one is interested in prescribing a treatment only to patients for whom it would be beneficial \cite{Jaskowski2012}.
Similarly in performance marketing, one would prefer to target advertisement budget towards potential customers that would be more likely to be persuadable to purchase \cite{Diemert2018}. A recent review of the Uplift Modeling literature \cite{devriendt2018literature,zhang2020unified} illustrates how these problems arise in a wide range of economic activities (credit scoring \cite{radcliffe1999differential}, catalog mailing, customer retention \cite{Radcliffe2007}, insurance \cite{guelman2014optimal}), medicine studies (bone marrow transplant, tamoxifen prescription, hepatitis \cite{Jaskowski2012} and breast cancer treatment \cite{kuusisto2014support}) and social sciences evaluations (job training programs \cite{imai2013estimating},  psychology \cite{kunzel2019metalearners} and student growth \cite{athey2019estimating}).

\textbf{Treatment assignment as a causality problem.} We formalize our goal as a problem of optimizing treatment assignment through individual treatment effect prediction.
From a causal perspective, assigning treatment to individuals is an \emph{intervention} and should be motivated by a causal model of the effect of treatment $T$ on the outcome of interest $Y$, conditionally on observed co-variates $X$.

Now there are two obvious possibilities to learn such a model: i) from data of a randomized experiment where the sole effect of treatment can be deduced, at the exclusion of all other causes or ii) from observational data and by assuming "unconfoundedness", that is we observe all causal parents of $Y$ and can adjust for them in the model.

Usually, the \UM literature relies on the former whereas Individual Treatment Effect (\ITE{}) studies assume the latter \cite{zhang2020unified}. 
In both cases a predictor of the \ITE{}: $ITE(x) = \E[Y|X=x,T=1] - \E[Y|X=x,T=0]$ can be learned from the data and corresponds to the difference of \emph{potential outcomes} in the Neyman-Rubin causal framework \cite{sekhon2008neyman}. As such it is able to accurately predict the benefit of assigning treatment to future individuals, given that the causal mechanism remains the same.

\textbf{Illustrative example.} Fig.~\ref{fig:treatment_optim} illustrates the \UM case, where data are available from prior, randomized experiments. It could be a pilot study using a randomized control trial with placebo for medicine or an A/B test for marketing (step 1). Such experiments can also be used to estimate the Average Treatment Effect (ATE) $ATE = \E[Y|T=1] - \E[Y|T=0]$. 
Then, different \ITE{} models can be learned and evaluated (step 2). A popular metric to value the quality of a model is the Area Under the Uplift Curve (\auuc) \cite{rzepakowski2010decision}. This metric measures the cumulative uplift along individuals sorted by predicted \ITE{}. A good model (with a high \auuc) scores higher those individuals for which the \ITE{} is high (beneficial) compared to ones for which the \ITE{} is low (neutral or even detrimental). Note that a perfect \ITE{} model would maximize \auuc{} \cite{Yamane2018}.
Finally, practitioners use \ITE{} predictions to \emph{rank} future instances and assign treatment to individuals with the highest scores (step 3) \cite{devriendt2018literature,fernandez2020methods}. So when a new cohort of individuals is available, the predictions of the model will be used to target treatment: highest scored individuals would be treated (green individuals in Figure \ref{fig:treatment_optim}) whilst the lowest scored ones would be excluded from treatment (blue individuals). This strategy is useful as soon as treatment effect is heterogeneous (i.e. depending on $X$).

\begin{figure}[!h]
    \centering
    \includegraphics[width=\linewidth]{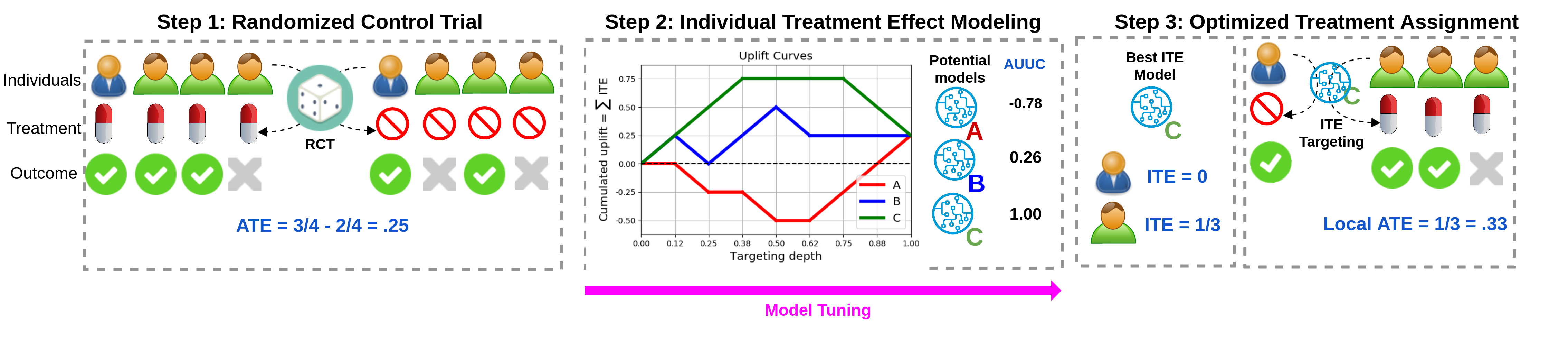}
    \caption{\textbf{The task of optimizing treatment assignment through \ITE{} Prediction}. Step 1 starts with a randomized control trial allowing to estimate ATE. Then Step 2 consists in learning and evaluating several \ITE{} models and selecting the best performing one by \auuc{} on data gathered at Step 1. Finally at Step 3, the best model is used to target treatment on the next cohort of individuals. Our main contribution is to propose an efficient model that generalizes safely between Step 2 and 3.\label{fig:treatment_optim}}
\end{figure}

\smallskip

\textbf{How ITE methods can fail?} As a motivation for our work we make the analogy with learning to rank. It is well known that classification algorithms designed to  minimize the error rate may not lead to the best possible Area Under the ROC Curve (AUC)   \cite{cortes2004auc}. For instance, one may arbitrarily reorder predictions greater than the classification threshold: this does not affect the error rate but leads to different AUC. 
Likewise, in \UM one may e.g. concurrently degrade the \ITE{} prediction of a sample and improve by the same amount another prediction: this operation does not affect the Precision when Estimating Heterogeneous Effect (PEHE) \cite{shalit2017estimating} but would lead to different \auuc{} values as the induced ranking changes.
In other words, when learning a classification (resp. \ITE{}) model one may inadvertently degrade AUC (resp. \auuc{}) whilst keeping a fixed error rate (resp. PEHE). In both cases this motivates the research of methods directly optimizing the metric of interest.

\smallskip

\textbf{Importance of generalization bounds.} Concurrently, works in machine learning now often incorporate generalization bounds in the design of learning algorithms \cite{langford2005tutorial}, a practice that developed recently in the \ITE{} and \UM community \cite{shalit2017estimating,Yamane2018}. 
Technically, for problems involving inter-dependent data samples (ranking being a noteworthy example) new bounds \cite{bartlett2005,usunier2006generalization,ralaivola2015entropy} have proven to be tighter than the classical MacDiarmid inequality \cite{mcdiarmid1989method}. This body of works inspired us to form a suitable bound for \auuc.

\smallskip

\textbf{Our contributions.} Considering the crucial role of treatment targeting in many applications, the need for models that optimize the metric of interest directly and the advances in the technical tools needed to study generalization properties of inter-dependent variables we form the following research agenda: i) study generalization bounds for \auuc, ii) derive a learning objective and iii) experiment the corresponding empirical performance compared to traditional methods.
Our main contributions in that respect are summarized as follows. 
\begin{itemize}
    \item[$(a)$] Using data-dependent concentration inequalities on dependent variables, Th. \ref{thrm:AUUC} states the first generalization bound for \auuc{} (Section \ref{sec:bounds}); 
\item[$(b)$] Algo. \ref{algo:auuc_max}  presents \auucmax, a linear \UMdl that optimizes a ranking formulation of \auuc{} guaranteed to generalize (Section \ref{sec:auuc_max_algo}); and finally 
\item[$(c)$] Section \ref{sec:experiments} reports a thorough performance evaluation against a wide range of competitive baselines. 
\end{itemize}

\section{Background} 
\label{ssec:eval_metrics}

\paragraph{Notations.} 
Let $\Xset \in \realset^d$, $\obs \in \Xset$ be a feature vector, and $y \in \{0, 1\}$ be the outcome variable, indicating positive ($y = 1$) or negative ($y = 0$) outcome. Additionally, let the treatment variable $g \in \{T, C\}$ denote whether an individual receives treatment ($g = T$) or not ($g~=~C$). We assume a dataset of size $N$ from a randomized control trial: $(x_i,g_i,y_i) \overset{\text{iid}}{\sim}
~P_{X,G,Y}; X \indep G$.
Also let $\tilde{S}^g = \{\obs_i^{1-y_i}, g\}_{i=1\ldots n^g}$ be the version of subset with reverted labels. 
We define then $S^g = \{\obs_i^{y_i}, g\}_{i=1\ldots n^g}$ as the particular subset of the training set $S$, i.e. $S=S^T\sqcup S^{C}$ and $N=n^T+n^C$; where $\obs_i^{y_i}$ represents the pair $(\obs_i,{y_i})$ composed by observation $\obs_i$ and its associated class $y_i$.
We define $\bar{y}^g = \mathbbm{E}[Y|T=g]$ and $\lambda^g = \bar{y}^g(1-\bar{y}^g)$ the treatment conditional mean and variance of Bernoulli outcome. Finally, let $\mathcal{F}=\{f:\Xset\rightarrow\realset\}$ be the set of real-valued functions. 

\paragraph{How to evaluate models?}
We will review here the formal definition of \auuc{} along with  basic concepts needed to frame the problem of maximizing \auuc. It is important to recall firstly the fundamental problem of causal inference: one can never observe a given individual in both treated and untreated conditions. Therefore metrics computing a difference to the true causal effect can only work in a simulation setting where you know both potential outcomes. A popular example of such a metric is Precision when Estimating Heterogeneous Effect (PEHE) \cite{shalit2017estimating}.
At the same time one can estimate group-level treatment effect by computing the so called local average treatment effect. This idea underlies the \emph{Area~Under~the~Uplift~Curve} (\auuc) \cite{rzepakowski2010decision}, the most popular method for evaluating \UMdl in the literature. Intuitively, with a good \ITE{} model, the best scored individuals should yield a high ATE. An uplift curve can verify this property: it ranks individual samples according to \emph{predicted} \ITE{} (see the X-axis of Fig. \ref{fig:treatment_optim}) and cumulatively sum the \emph{observed} ATE (in Y-axis). The \auuc{} is then the area under this curve. On Fig.~\ref{fig:treatment_optim} - middle picture, \ITE{} model \textcolor{green}{C} is better than \textcolor{red}{A} or \textcolor{blue}{B} as its area is greater, denoting that cumulated ATE is greater when individuals are ranked according to model \textcolor{green}{C}. 
\auuc{} is thus useful when $i)$ all individuals could not be "treated" (e.g. because of limited budget) and $ii)$ data comes from real-world (no simulation). Intuitively, \auuc{} penalizes heavily ranking errors on highest scored individuals, which is reasonable for settings where treatment would be administered only to highest predictions in future cohorts. 

\paragraph{Formalization of \auuc{}.} The uplift modeling literature yields multiple estimators of \auuc, with differences residing mainly in $i)$ the way treatment imbalance is accounted for; and $ii)$ whether treated and control groups are ranked separately or jointly. Readers can refer to Table 2 in \cite{devriendt2020learning} for a comprehensive picture of available alternatives. We chose the "joint, relative" estimator introduced in (Eq. 13) of \cite{surry2011quality}. Evaluations in  \cite{devriendt2020learning} have concluded that this choice is robust to treatment imbalance and captures well the intended usage of \ITE{} models to target future treatments. We give a self-contained formula in Def. \ref{def:auuc}, corresponding to (Eq. 12 and 15) of \cite{devriendt2020learning}.


\begin{definition}[Area Under the Uplift Curve]
Let $f(\mathcal{D},k)$ be the $k$ first elements of the dataset $\mathcal{D}$ when ordered by prediction of model $f$. Let $|T|$ (resp. $|C|$) be the number of treated, $T=1$ (resp. control, $T=0$) individuals in the dataset. The joint, relative \auuc{} of model $f$ is given by:
\label{def:auuc}
\begin{equation*}
AUUC(f) = \int_0^1 V(f,x)dx \approx \sum_{k=1}^n V(f,k) 
\end{equation*}
where 
\begin{equation*}
V(f,k) = \frac{1}{|T|} \sum_{i \in f(\mathcal{D},k) } y_i \indicatrice_{[t_i = 1]} - \frac{1}{|C|} \sum_{j \in f(\mathcal{D},k) } y_j \indicatrice_{[t_j = 0]}
\end{equation*}
\end{definition}

\paragraph{Basic models.} We now introduce popular \ITE{} prediction models used for \UM{} to materialize the task. \emph{Two Models (TM)} \cite{hansotia2002incremental} is a trivial method to predict \ITE. It uses two separate probabilistic models to predict outcome in treated or untreated conditions:

\vspace{-1.5em}
\begin{equation}
\label{eq:twomodels}
P^{TM}(x) = P(Y=1 | X=x, T=1) - P(Y=1 | X=x, T=0)
\end{equation}

and any prediction model can be used (typically logistic regression). We notice that when the average response is low and/or noisy there is the risk for the difference of predictions to be very noisy too and lead to arbitrary ranking of individuals overall (see \cite{radcliffe2011real} for a detailed critic). 
This remark makes a general argument for using methods that combine knowledge of both parts of the dataset. 
Multi-task approaches, e.g. \emph{Shared Data Representation} (\emph{SDR}) have been proposed in \cite{betlei2018uplift} to overcome this problem and showed better empirical performance when the treatment is imbalanced.
\emph{Class Variable Transformation (CVT)} \cite{Jaskowski2012} combines binary treatment and outcome in order to use a single classification model. For this purpose a new label and predictor are defined:
\begin{equation}
\label{eq:Reverted}
\begin{cases*}
& $Z = Y T + (1-T)(1-Y)$\\ 
& $P^{CVT}(x)  = 2 P(Z = 1|X=x) - 1$
\end{cases*}
\end{equation}
Similar label transformations could be traced back to Robinson \cite{robinson1988root} and extended to more general settings \cite{nie2017quasi,athey2015machine}. Other related, productive lines of research have been $i)$ the adaptation of split criteria of Decision Trees \cite{radcliffe2011real,rzepakowski2012decision,Sotys2015} for \ITE{}  prediction and; $ii)$ deep representation learning approaches for the observational case that carefully match treated/control embedding distributions \cite{shalit2017estimating,bica2020estimating,louizos2017causal}.

\section{On the Generalization Bound of AUUC and Learning Objective} 
\label{sec:bounds}
In this section, we bound the difference between \auuc{} and its expectation and use this new bound to formulate a corresponding learning objective.
For that purpose, we start by drawing a connection between \auuc{} and bipartite ranking risk (Section \ref{ssec:ConnAUUC_AUC}); and by means of Rademacher concentration inequalities  build a bound (Section \ref{ssec:RadGenBnd}). Then we define a principled optimization method with generalization guarantees for \auuc{} that leverages the bound as a robust learning objective (Section \ref{sec:auuc_max_algo}). Finally, we review related approaches and their merits as found in the literature (Section \ref{ssec:relatedworks_discussion}).

\subsection{Connection between AUUC and Bipartite Ranking Risk}
\label{ssec:ConnAUUC_AUC}

From the equivalence between the Area under the ROC curve and the bipartite ranking risk, we can show that \auuc{} is  a weighted combination of ranking losses for the treatment and control responses. Formal version of the decomposition is provided in Proposition \ref{prop:one}.

\begin{pro}
\label{prop:one}
Let $AUUC(f,S^T,S^C)$ be the empirical area under uplift curve of the model $f$ on the sets $S^T$ and $S^C$; and $AUUC(f)=\mathbb{E}_{S^T,S^C} \left[AUUC(f,S^T,S^C)\right]$ be its expectation. Then $AUUC(f)$ is related to ranking loss (Eq. \ref{eq:auc}) as:
\begin{equation}
\begin{aligned}
\label{eq:auuc_risks} 
AUUC(f) = \gamma-\lambda^T\mathbb{E}_{S^T}[\hat{\risk}(f,S^T)]-\lambda^C\mathbb{E}_{\tilde{S}^C}[\hat{\risk}(f,\tilde{S}^C)]
\end{aligned}
\end{equation}
where
\begin{equation}
	\hat{\risk}(f,S^g) \triangleq \frac{1}{n^g_+n^g_-}\sum_{(\obs_i,+1)\in S^g}\sum_{(\obs_j,0)\in S^g}\indicatrice_{f(\obs_i)\leq f(\obs_j)}
	\label{eq:auc}
\end{equation}
is the empirical bipartite ranking risk, $g\in\{T,C\}$, $n^g_+,n^g_-$ are the amounts of positives and negatives respectively in the set $S^g$ (i.e. $n^g=n^g_-+n^g_+$), and $\gamma = \bar{y}^T - \frac{(\bar{y}^T)^2}{2} - \frac{(\bar{y}^C)^2}{2}$. 
\end{pro}

\textit{Sketch of proof}. Our derivation can be divided into two steps. The first step is to decompose \auuc{} into the weighted difference of Areas Under ROC curve (AUC) for the groups $T$ and $C$ using properties from \cite{surry2011quality} and \cite{tuffery2011data}. Then we revert the labels in control group in order to turn our decomposition from a weighted difference to a weighted sum to be able to construct a union bound of two AUCs. The full proof is given in the supplementary.



\subsection{Rademacher Generalization Bounds}
\label{ssec:RadGenBnd}
Let us now consider the minimization problems of the pairwise ranking losses over the treatment and the control subsets (Eq. \ref{eq:auc}), and the following dyadic transformation defined over each of the groups~$S^T$~and~$\tilde{S}^C$:
\begin{equation*}
\label{eq:transfo}
\mathcal{T}(S^g)\! = \left\{\left(\bfZ=(\obs,\obs'),\tilde{y} \right)\! \bigm\vert\!  (\obs^y,\obs'^{y'})\!\in\! S^g\!\times  S^g  \wedge y\neq y'\right\}
\end{equation*}
where $g\in\{T,C\}$, and $\tilde y=+1$ iff $y=+1$ and $y'=0$ and; $\tilde y=-1$ otherwise. Here we suppose that $\mathcal{T}(S^g)$  contains just one of the two pairs that can be formed by two examples of different classes. This transformation corresponds then to the set of $n_+^gn_-^g$ pairs of observations in $S^g$ that are from different classes. 

From this definition and  the class of functions, $\mathcal H$, defined as:
\begin{equation}
\label{eq:ClassOfFunction}
\mathcal H=\{h:  \bfZ=(\obs,\obs') \mapsto  f(\obs) - f (\obs')  , f\in\mathcal{F}\},
\end{equation}
the empirical loss (Eq.~\ref{eq:auc})  can be rewritten as:
\begin{equation}
\label{eq:EmpLossMC2}
\hat{\risk}(h,\mathcal{T}(S^g)) =\frac{1}{n^g_+n^g_-}\sum_{(\bfZ,\tilde{y})\in \mathcal{T}(S^g)}\indicatrice_{\tilde{y} h(\bfZ)\leq 0}.
\end{equation}
The loss defined in (Eq.~\ref{eq:EmpLossMC2}) is equivalent to a binary classification error over the pairs of examples in $\mathcal{T}(S^g)$. With this equivalence, one may expect to use efficient generalization bounds developed in binary classification.   However, (Eq. \ref{eq:EmpLossMC2}) is a sum over random dependent variables; as each training examples in $S^g$ may be present in different pairs of examples in $\mathcal{T}(S^g)$, and the study of the consistency of the Empirical Risk Minimization principle cannot be carried out using classical tools; as the  central i.i.d. assumption on which these tools are built on is transgressed. For this study, we consider $\mathcal{T}(S^g)$ as a dependency graph of random variables $(\bfZ_{j},\tilde{y}_j)_j$ on its nodes, and similar to \cite{usunier2006generalization}, we decompose it using the \textit{exact proper fractional cover} of the graph proposed by \cite{Janson04RSA} and defined as:

\begin{definition}

Let $\graph=(\vertices,\edges)$ be a
graph. $\cover~=~\{(\Cset_j,\omega_j)\}_{j\in[J]}$, for some positive integer $J$, with
$\Cset_j\subseteq\vertices$ and $\omega_j\in [0;1]$ is an exact proper
fractional cover of $\graph$, if:
\begin{enumerate}
\item it is {\em proper:} $\forall j,$ $\Cset_j$ is an {\em independent set}, i.e., there is no connections between vertices in
  $\Cset_j$;
\item it is an {\em exact fractional cover} of $G$: $\forall
  v\in~\vertices,\;\sum_{j:v\in\Cset_j}\omega_j= 1$.
\end{enumerate}
The weight $\Weight(\cover)$ of $\cover$ is given by: $\Weight(\cover)=\sum_{j\in[J]}\omega_j$ and the
minimum weight $\chi^*(\graph)=\min_{\cover\in\covers(\graph)} \Weight(\cover)$ over the set $\covers(\graph)$ of all exact proper fractional covers of $\graph$ is the {\em fractional chromatic number} of $\graph$.
\end{definition}

Here, the weight $\Weight(\cover)$ of $\cover$ is given by $\Weight(\cover)=~\sum_{k=1}^J \omega_k$ and the minimum weight, called the fractional chromatic number, and defined as $\chi^*(\graph)=\min_{\cover\in\covers(\graph)}\Weight(\cover)$ corresponds to the smallest number of subsets containing independent variables. 


From this definition, \cite{ralaivola2015entropy} proposed new concentration inequalities by extending the fractional Rademacher complexity introduced in \cite{NIPS2005_2860} to local fractional Rademacher complexity \cite{bartlett2005} defined by a bound over the variance of the prediction functions. In this case, a strategy which consists in choosing a model with the best generalization error tends to select functions with small variance in their predictions and a small bounded complexity. 
\begin{definition}

\label{def:lfrc}
The Local Fractional Rademacher Complexity, $\rademacher_{S^g}(\Fset_r)$, of the class of functions with bounded variance $\Fset_{r}=\{f:\Xset\mapsto \mathbb{R}:\variance f\leq r\}$ over the dyadic transformation, $\mathcal{T}(S^g)$ of size $n^g_+n^g_-$, of the set $S^g$, is given by:
\begin{equation}
\label{eq:LFR}
	\rademacher_{S^g}(\Fset_r)\!
	=\!\frac{1}{n^g_+n^g_-}\expectation_{\sigma}\!\!\left[\sum_{j\in[J]}\omega_j\expectation_{X_{\Cset_j}}\!\!\left[\sup_{f\in\Fset_r}\sum_{i\in\Cset_j}\sigma_if(X_i)\right]\right]
\end{equation}
with $\mathbf{\sigma}=(\sigma_1,\ldots,\sigma_{n_+^gn_-^g})$ being $n_+^gn_-^g$ independent Rademacher variables verifying:\\ $\proba(\sigma_i=+1)=\proba(\sigma_i=-1)=1/2; \forall i\in\{1,\ldots,n_+^gn_-^g\}$.
\end{definition}

From these statements, we can now present the first data-dependent generalization lower bound for \auuc.

\begin{theo}
\label{thrm:AUUC}
        Let $S=(\obs_i^{{y_i}})_{i=1}^m\in (\Input\times \Output)^m$ be a dataset of $m$ examples drawn i.i.d. according to a probability distribution $\Dist$ over $\Input\times \Output$,  and decomposable according to treatment $S^T$ and reverted label control $\tilde{S}^C$ subsets. Let  $\mathcal{T}(S^T)$ and $\mathcal{T}(\tilde{S}^C)$ be the corresponding transformed sets. Then for any $1>\delta>0$ and $0/1$ loss $\ell:\{-1,+1\}\times\mathbb{R}\rightarrow [0,1]$, with probability at least~$(1-\delta)$ the following lower bound holds for all $f\in \mathcal \Fset_{r}$:
        \begin{align*}
        AUUC(f)\!\geq &~\gamma - \Bigl(\lambda^T \hat{R}_\ell(f,S^T)+\lambda^C \hat{R}_\ell(f,\tilde{S}^C)\Bigr)\nonumber 
        - C_\delta(\Fset_r,S^T,\tilde{S}^C) - \frac{25}{48}\biggl(\frac{\lambda^T}{n_+^T}+\frac{\lambda^C}{n_\_^C}\biggr)\log\frac{2}{\delta}
        \end{align*}
{where, \footnotesize $ C_\delta(\Fset_r,S^T,\tilde{S}^C)\!=\!(\lambda^T\rademacher_{S^T}(\Fset_r))+\lambda^C\rademacher_{\tilde{S}^C}(\Fset_r))+ \!\left(\frac{\!\frac{5}{2}\sqrt{\rademacher_{S^T}(\Fset_r)}+\frac{5}{4}\sqrt{2r}}{\sqrt{n_+^T}}\lambda^T\!\!\!+\!\!\frac{\frac{5}{2}\sqrt{\rademacher_{\tilde{S}^C}(\Fset_r)}+\frac{5}{4}\sqrt{2r}}{\sqrt{n_\_^C}}\lambda^C\right)\!\!\sqrt{\log\frac{2}{\delta}}$} is defined with respect to local Rademacher complexities of the class of functions $\mathcal{F}_r$ estimated over the treatement and the control sets.
        \end{theo}
\textit{Sketch of proof}. The proof is based on the generalization upper bounds of the ranking losses, $\hat{R}_\ell(f,S^T)$ and $\hat{R}_\ell(f,\tilde{S}^C)$, proposed in \cite{ralaivola2015entropy}. The result is then deduced from the union bound after finding the optimal constants that appear in the infimums of these generalization bounds. Detailed proof is provided in the supplementary material.

Note that the convergence rate of the bound is governed by least represented class in both treatment and reverted control subsets. 

\subsection{AUUC-max Algorithm} 
\label{sec:auuc_max_algo}


From Theorem \ref{thrm:AUUC}, we can formulate an optimization problem for the expected value of \auuc{} as follows:
\begin{align}
\label{eq:auuc_loss_raw} 
    \argmax_{f \in \Fset_{r}}  AUUC(f) \equiv \argmin_{\theta, r} \left(\lambda^T \hat{\risk}(f_\theta,S^T) + \lambda^C \hat{\risk}(f_\theta,\tilde{S}^{C}) + C_\delta(\Fset_r,S^T,\tilde{S}^C)\right)
\end{align}  

There are two remarks that we can make at this point. First, both terms $\hat{\risk}(f_\theta,S^T)$ and $\hat{\risk}(f_\theta,\tilde{S}^{C})$ in \eqref{eq:auuc_loss_raw}  are defined over the instantaneous ranking loss $\indicatrice_{\tilde{y}(f(\obs)-f(\obs'))\leq 0}$ and in practice we need a differentiable surrogate over these losses so that the minimization problem can be solved using standard optimization techniques. Second, the local fractional Rademacher complexities $\rademacher^T(\Fset_{r})$ and $\rademacher^C(\Fset_{r})$  that appear in $C_\delta(\Fset_r,S^T,\tilde{S}^C)$ should be estimated for some fixed class of functions $\Fset_{r}$ with a well suited value of $r$. 

For the first point, we propose to use differentiable surrogates of the instantaneous ranking loss, such as $ s_{log}\left(z\right) = \ln \left(1 + e^{-z}\right) / \ln(2)$ and $s_{poly}\left(z\right)=~\left(-\left(z-\mu\right)\right)^{p} \indicatrice_{z<\mu}$. Note that $s_{log}$ upper-bounds the indicator function $\indicatrice_{z \leq 0}$.  This is also the case for $s_{poly}$  with $\mu = 1$ and $p = 3$.

For the second point, we propose to use the class of linear functions and to  upper bound both local Rademacher complexities $\rademacher^T(\Fset_{r})$ and $\rademacher^C(\Fset_{r})$ following Proposition \ref{prop:two}.

\begin{pro}
\label{prop:two}
    Let $S^g$ be a sample of size $n^g$ with $n^g_+$ samples with positive labels and such that $\forall \obs\in~S^g; \|\obs\|\leq~R
$. Let $\Fset_{r}=\{\mathbf{x} \mapsto \mathbf{w}^\top \mathbf{x}:\|\mathbf{w}\| \leq \Lambda ; f\in \Fset:\variance f\leq r\}$, be the class of linear functions with bounded variance and bounded norm over the weights. Then for any $1>\delta>0$, the empirical local fractional Rademacher complexity of $\Fset_{r}$ over the set of pairs $\mathcal{T}(S^g)$ of size $n_+^gn_-^g$, can be bounded with probability at least $1-\frac{\delta}{2}$ by: \\
    \begin{equation}
    \label{eq:R_LF_bound} 
    \rademacher_{S^g}(\Fset_{r}) \leq \sqrt{\frac{R^{2} \Lambda^{2}}{n^g_+}} + \sqrt{\frac{\log \frac{2}{\delta}}{2n^g_+}}
    \end{equation}

\end{pro}

\textit{Sketch of proof}. The proof is based on the linearity of expectation that appears in the definition of $\rademacher^S(\Fset_{r})$ (Def. \ref{def:lfrc}) and follows from Equation 3.14 defined in \cite[p.~36]{mohri2018foundations} on the connection between true and empirical Rademacher complexities, and Theorem 4.3 \cite[p.~77]{mohri2018foundations}. The full proof is given in the supplementary material. 
Remark that this bound could be extended to non-linear models using a appropriate estimator of $\rademacher^{S^g}(\Fset_{r})$ \cite{barron2019complexity}.

Finally, we apply Cauchy-Swartz and Popoviciu's inequalities to bound the variance of any function $f\in\Fset_{r}$, $\variance f$,  by $r = \Lambda^2 R^2$. Noting that $R$ is a constant of the dataset we can transform the optimization problem in $(\theta, r)$ in (Eq. \ref{eq:auuc_loss_raw}) to a problem in $(\mathbf{w}, \Lambda)$. Furthermore, the constraint on the weights $\Lambda$ can be considered in practice as a max-norm regularizer \cite{srivastava2014dropout} and considered as a hyperparameter of the model.

From these settings, and the definition of a surrogate loss over the instantaneous ranking loss $s:\mathbb{R}\rightarrow \mathbb{R}_+$, the version of the optimization problem (\ref{eq:auuc_loss_raw}) that we consider is now:

\begin{centering}
\resizebox{0.95\hsize}{!}{%
\begin{tcolorbox}[top=0pt,left=0pt,right=0pt,bottom=0pt,colback=gray!5!white,colframe=gray!80!white,title={\auucmax{} optimization problem},fontupper=\small]
\begin{small}
\begin{align}
\begin{split}
\label{eq:auuc_loss}
& \tiny{\min_{\mathbf{w}}  \hat{\mathcal{L}}_{\mathbf{w}}(S^T,\tilde{S}^C) = \frac{1}{n^T_+n^T_-}\!\!
\sum_{\obs_i^{+1}\in S^T} \!\sum_{\obs_j^0\in S^T} s(\mathbf{w}^\top \!\mathbf{x}_i-\mathbf{w}^\top \!\mathbf{x}_j)} \\
 & {\footnotesize + \frac{1}{n^C_+n^C_-} \!\!\sum_{\obs_k^{+1}\in \tilde{S}^C} \!  \sum_{\obs_l^0\in \tilde{S}^C} \!\!s(\mathbf{w}^\top\! \mathbf{x}_k\!-\!\mathbf{w}^\top\! \mathbf{x}_l)\! +\!  C_\delta(\Fset_{\Lambda^2R^2},\!S^T\!\!,\!\tilde{S}^C)} \\
& ~~~~~~~~~~~~~~~~~~~~\textrm{subject to} \quad  \|\mathbf{w}\| \leq \Lambda\\
\end{split}
\end{align} 
\end{small}
\end{tcolorbox}
}
\end{centering}

      \begin{algorithm}[t!]               
        \caption{\textbf{Linear \auucmax}}
              \textbf{Input:} Train set $S$, penalty grid $G_\Lambda$, optimizer grid $G_\eta$, surrogate $s: \mathbbm{R} \mapsto \mathbbm{R}_+$\\
            \ForAll{$\Lambda, \eta \in (G_\Lambda, G_\eta)$}{
                $\mathbf{w} \leftarrow Optimize(\hat{\mathcal{L}}(S^T,\tilde{S}^C), \eta, s)$ \\
                $b_{\mathbf{w},\Lambda} \leftarrow AUUCLowerBound(\mathbf{w}, \Lambda, S)$ \\
                $\mathbf{w}^* \leftarrow \mathbf{w}_{\Lambda,\eta}$ and $\Lambda^* \leftarrow \Lambda$ iff $b_{\Lambda,\eta} > b_{\mathbf{w}^*, \Lambda^*}$\\
           }
        \textbf{Output:} final weights $\mathbf{w}^*$ 
            \label{algo:auuc_max}
\end{algorithm}

The derived algorithm minimizing loss (\ref{eq:auuc_loss}) is called \auucmax{} and its pseudo code is presented in Algorithm \ref{algo:auuc_max}. At a high level we decompose the optimization problem in $(\mathbf{w}, \Lambda)$ of (Eq. \ref{eq:auuc_loss}) by choosing a grid of values for $\Lambda$ and make use of the generalization guarantees of the bound to select the best model $\mathbf{w}^*$ by its lower bound $b_{\mathbf{w}^*, \Lambda^*}$. Theoretically, a joint or alternate optimization over $(\mathbf{w}, \Lambda)$ is also possible. Interestingly, a small grid $G_\Lambda$ is sufficient in practice to obtain competitive performance (see Section \ref{sec:experiments}).

Note that the usual practice for \ITE{} models (see Fig.~\ref{fig:treatment_optim}) is to iterate over parameter grids (e.g. for optimization and regularization) and select the best model by estimating the mean empirical \auuc{} over a $k$-fold cross-validation: this implies an inner "for" loop in place of $AUUCLowerBound()$ and additional computations.

\subsection{Related Works}
\label{ssec:relatedworks_discussion}
In this section, we review some related works that address the problems of \auuc{} maximization and the generalization study of \ITE{}.

\textbf{SVM for Differential Prediction} \cite{kuusisto2014support} proposes to maximize \auuc{} directly by expressing it as a weighted sum of two AUCs and maximizing it using a Support Vector Machine objective. 
Our work bears similarity to their seminal work by borrowing the idea of decomposing \auuc{} into a weighted sum of AUCs. We further propose to optimize differentiable surrogates of the objective in the case of imbalanced treatment, and provide generalization bounds as well as an efficient hyperparameter tuning procedure.

\textbf{Promoted Cumulative Gain} \cite{devriendt2020learning} draw a learning to rank formulation of \auuc{} similar to ours and use the LambdaMART \cite{burges2010ranknet} algorithm to optimize it, alleviating the need for derivable surrogates at the price of more complex models.

\textbf{Representation learning for \ITE{} prediction} Important work has been published recently using a broad family of methods to perform \ITE{} prediction in the observational case (no randomized treatment and therefore selection bias): CFRNet \cite{shalit2017estimating}, CRN \cite{bica2020estimating}, CEVAE \cite{louizos2017causal}, BNN \cite{johansson2016learning}, GANITE \cite{yoon2018ganite}, DeepMatch \cite{kallus2018deepmatch}. Typically these methods revolve around ways to lessen the covariate shift between $P_{X|T=1}$ and $P_{X|T=0}$ due to treatment selection bias. E.g. CFRNet and BNN add a penalty in the loss to control distribution distance between embeddings $\phi(P_{X|T=1})$ and $\phi(P_{X|T=0})$, CEVAE models hidden confounders that biases treatment. 
Then in our setting we already have $X \!\perp\!\!\!\perp T$ in the learning data by design so we expect that these methods perform comparably to regular \ITE{} models.

\textbf{Generalization bounds} The work of \cite{shalit2017estimating} provide a bound for the PEHE metric (so usable for simulation settings) and pioneered the use of generalization bounds for \ITE{} prediction. More closely to our work; \cite{Yamane2018} proposed a generalization bound for uplift prediction. However, the main differences with our approach is that the upper-bound of \auuc{} proposed in \cite{Yamane2018} is a MSE-like proxy that is applicable in the case where the variables $Y$ and $T$ are never observed together whereas we bound \auuc{} directly without such hypothesis. Further, the definition of the proxy objective proposed in \cite{Yamane2018} assumes that samples are i.i.d.; whilst in our study the equivalence between the ranking objective \eqref{eq:auc} and the classification error over the pairs of examples \eqref{eq:EmpLossMC2} gives rise to the consideration of dependent samples  that calls for specific concentration inequalities, namely \emph{local, fractional} Rademacher theory, that ensures fast convergence rates \cite{bartlett2005}. Finally, from an optimization perspective the approach developed in \cite{Yamane2018} leads to a mini-max optimization problem, that is avoided in \auucmax{} by using the "revert label in control" trick. 

\section{Experimental setup and results} 
\label{sec:experiments}
In this section, we provide comparisons between the proposed approach and the state-of-the-art models.

\textbf{Experimental Setting}. We use two open, real-life datasets. 
\emph{Hillstrom} \cite{Hillstrom2008} contains results of an e-mail campaign for an Internet retailer. Treatment (receiving promotional e-mails) is balanced and independent of co-variates; outcome is visiting the retailer website. This dataset is a classical benchmark for \UM and uses the \auuc{} performance metric (Eq. \ref{eq:auc}).
For each method (except PCG/NDCG from \cite{devriendt2020learning}) we repeat (training, then hyperparameter selection on validation then testing) over 100 random train+validation/test splits; hyperparameter grids for each method are of the same size. For MML we follow \cite{Yamane2018} and draw separate samples as is prescribed in their empirical study. Experimentation code and more details are provided in the supplementary.
\emph{Jobs} \cite{lalonde1986evaluating} is a classical benchmark for \ITE{} prediction. It is composed of observational and randomized data on the effect of job training on income and employment.
For this task we follow closely the experimental setting described in \cite{shalit2017estimating} and use the \emph{policy risk} metric:
\begin{equation}
\label{eq:policy_risk}
\mathcal{R}_{\mathrm{pol}}(\pi_{f}, \alpha)= 1 - \mathbb{E}[Y \mid T=1;\pi_{f}(x)>\alpha] \cdot P(\pi_{f}>\alpha) - \mathbb{E}[Y \mid T=0;\pi_{f}(x) \leq \alpha] \cdot P(\pi_{f}\leq \alpha)
\end{equation}
to measure the risk to treat a ratio $\alpha$ of the population based on a policy $\pi_f$ that selects individuals with highest predictions from model $f$.

\textbf{Local, fractional Rademacher gives tightest bounds.}
As a preliminary we examine the choice of local, fractional Rademacher theory to form a bound. For that purpose we compute the generalization error on the \emph{Hillstrom} dataset for different variants of Th. \ref{thrm:AUUC}: using the local, fractional Rademacher concentration inequality on AUC from \cite{ralaivola2015entropy} (our proposition, in \textcolor{blue}{blue} on Fig.~\ref{fig:auuc_bounds}) or \cite{agarwal2005generalization,Usunier:1121} (in \textcolor{orange}{orange}) or \cite{freund2003efficient} (in \textcolor{green}{green}). We observe that our bound makes an average error of $0.02$, which is much tighter than the alternatives. This result illustrates the benefit of a data-dependent analysis framework for dependent variables.

\begin{figure}[t]
    \centering
    \includegraphics[width=.6\linewidth]{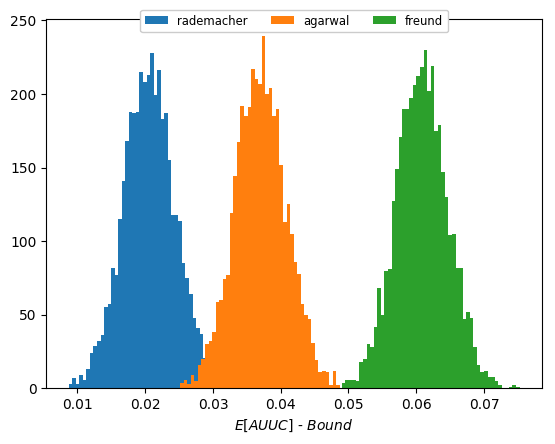}
    \captionsetup{width=.95\linewidth}
    \caption{
        \textbf{Gap between $\E[AUUC]$ and different versions of the bound} (closer to $0$ is better).
    }
    \label{fig:auuc_bounds}  
\end{figure}

\textbf{Polynomial surrogate performs best}. 
The bottom lines of Table \ref{tab:auuc_perf_summary} show the performance of different \auucmax{} variants for the $s_{poly}$ and $s_{log}$ surrogates. For both the cross-validation or bound versions the $s_{poly}$ surrogate is giving better results, indicating it should be the default choice.

\begin{table}[t]
\captionof{table}{\textbf{Test \auuc{} on \emph{Hillstrom}}: comparison of baselines and \auucmax{}. (\textbf{-}): significantly lower than \auucmax{} ($s_{poly}$) by binomial test at 95\%. 
}
\begin{center}
    \label{tab:auuc_perf_summary}
    \begin{tabular}{l|l|r|r}
    Model & Test \auuc{} $\pm2\sigma$& \#params & Time\\
    \midrule
    TM (Eq. \ref{eq:twomodels}) &  .03019 $\pm$ .00652 (\textbf{-})& 46 & 1.00x\\
    CVT (Eq. \ref{eq:Reverted}) &  .03035 $\pm$ .00599 (\textbf{-})& 23 & 0.53x\\
    SVM-DP \cite{kuusisto2014support}  &  .03047  $\pm$ .00606 (\textbf{-})& 23 & 0.49x\\
    DDR \cite{betlei2018uplift}  &  .03042 $\pm$ .00611 (\textbf{-})& 47 & 1.31x\\
    SDR \cite{betlei2018uplift}  &  .03079 $\pm$ .00633 & 67 & 1.10x\\
    MML \cite{Yamane2018} & .02222 $\pm$ .00869 (\textbf{-}) & 46 & >10x\\
    TARNet \cite{shalit2017estimating}  & .03044 $\pm$ .00626 (\textbf{-}) & 22,402 & 6.00x\\
    GANITE \cite{yoon2018ganite}  & .02916 $\pm$ .00818 (\textbf{-}) & 7,045 & > 10x \\
    PCG \cite{devriendt2020learning} & .03063 $\quad \quad$ N/A & $\approx$ 5,000 & N/A\\
    NDCG \cite{devriendt2020learning} & .02954 $\quad \quad$ N/A & $\approx$ 5,000 & N/A\\
    \midrule
    \auucmax ($s_{poly}$) + CV & .03071 $\pm$ .00608 & 23 & 0.66x\\
    \auucmax ($s_{poly}$) & .03065 $\pm$ .00612 & 23 & 0.17x\\
    \auucmax ($s_{log}$) + CV & .03023 $\pm$ .00619 (\textbf{-})& 23 & 0.69x \\
    \auucmax ($s_{log}$) & .03024 $\pm$ .00614 (\textbf{-})& 23 & 0.18x \\
    \end{tabular}
\end{center}
\end{table}

\textbf{Tuning parameters by bound is efficient}.
\label{ssec:model_sel}
Following \cite{ShaweTaylorBounds07} we compare \auucmax{} with hyperparameters chosen by bound versus chosen by cross-validation (+CV) in Table \ref{tab:auuc_perf_summary}. Models tuned by either methods are practically equivalent (up to the 4th digit) whilst the bound method yields computation savings in $\mathcal{O}(k)$ where $k$ is the number of folds.

\textbf{\auucmax{} is competitive in practice}.
Table \ref{tab:auuc_perf_summary} contains quantitative performance results of \auucmax{} and a large selection of competitive baselines on \emph{Hillstrom}. 
First we remark that, in line with previous studies \cite{Diemert2018, devriendt2020learning, kuusisto2014support, Jaskowski2012}, it is difficult to observe statistically significant results on this task. Nonetheless, small increases in \auuc~can lead to important gains in the application \cite{radcliffe2011real}.
We note that \auucmax ($s_{poly}$) ranks 2nd, indicating that our method is competitive. It is only beaten by SDR, a multi-task \ITE{} method.


\textbf{Model complexity and computing efficiency}. Fig. \ref{fig:auuc_perf_hillstr} represents the relation between model complexity and performance for the best \auucmax{} variant and baselines on \emph{Hillstrom} dataset. We notice that our method is the simplest to give as good results, only outperformed by SDR (3x more parameters). PCG yields slightly lower performance with 200x more parameters. Moreover, when comparing the training time in the last column of Table \ref{tab:auuc_perf_summary} (relative to TM) it is striking that our method achieves the smallest training time overall while performing second. SDR is 6 times slower.
\begin{figure}[h!]
    \centering
    \includegraphics[width=.6\linewidth]{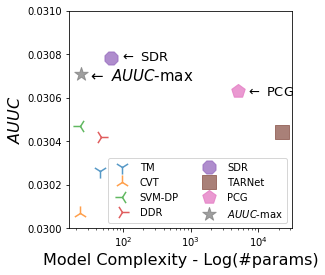}
    \captionof{figure}{{\textbf{Test \auuc{} vs Model complexity}} on \emph{Hillstrom}. (higher \auuc{}, lower complexity is better)}
    \label{fig:auuc_perf_hillstr}
\end{figure}


\begin{figure}[h!]
    \centering
    \includegraphics[width=.6\linewidth]{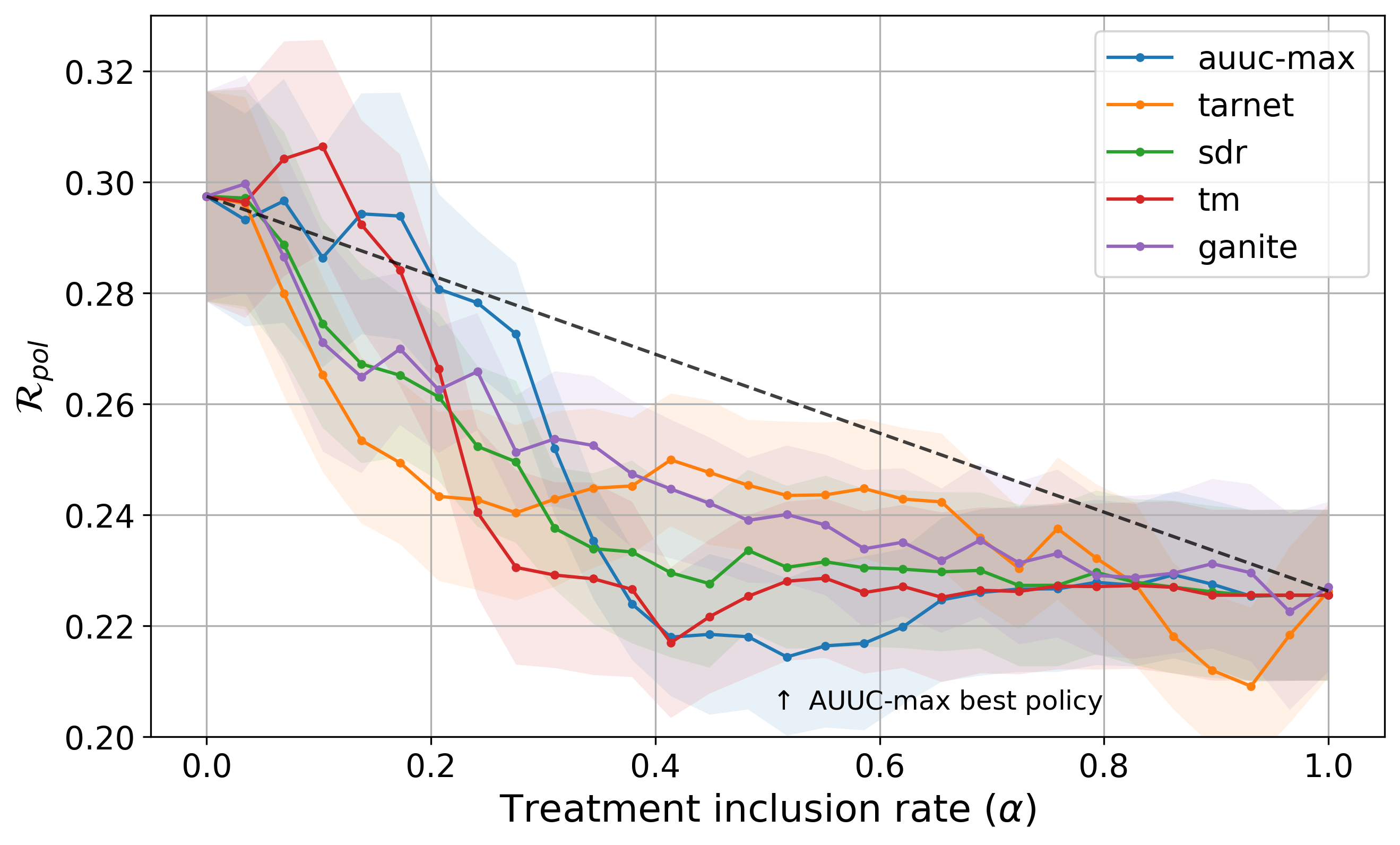}
    \captionof{figure}{{\textbf{Policy Risk}} on Jobs. (lower is better)}
    \label{fig:pr_perf_summary}
\end{figure}

\textbf{\auucmax{} is competitive on related tasks}. 
If \auucmax{} is indeed able to target treatment for maximum individual benefit we should expect that it would be competitive on an alternative metric: \emph{policy risk} (Eq. \ref{eq:policy_risk}), even though the method is not optimizing it directly.
Results on Fig.~\ref{fig:pr_perf_summary} show that \auucmax{} is able to find an efficient treatment assignment policy for $\alpha \approx .5$ (indicated by $\uparrow$ on the figure) and yields less risk than most baselines. Only TARNet manages to find policies with a similar risk but for larger $\alpha$, meaning when treating more individuals in the population.
Now remark that it is often costly to treat a large portion of the population. So the best policy found by TARNet at $\alpha \approx .95$ is probably unfeasible in practice and needs to treat 40\% more individuals to reach the same risk. In that light we shall conclude that \auucmax{} finds the best usable policy for most applications.

\section{Conclusion and future works}
\label{sec:conclusion}
We propose the first, data-dependent generalization lower bound for the treatment assignment optimization metric, \auuc, used in numerous practical cases. 
Then we derive a robust learning objective that optimizes a derivable surrogate of the \auuc{} lower bound. 
Our method alleviates the need of cross-validation for choosing regularization and optimization parameters, as we empirically show. As a result we highlight its simplicity and computational benefits. 
Experiments show that our method is competitive with the most relevant baselines from the literature, all methods being properly and fairly tuned. 
An exciting area for future works would be to research a version of the bound usable for deep models. In light of our experiments we conjecture that adding a representation learning component to \auucmax{} could improve performance beyond what is possible today while keeping good properties: generalization guarantees and hence thrifty use of computing resources.

\bibliographystyle{unsrt}  
\bibliography{neurips}  






\newpage
\section*{Appendix}

\subsection*{Proofs}
Here we provide the proofs of Proposition 1, Theorem 1 and Proposition 2, stated in the paper.

\begin{pro}
\label{prop:one}
Let $AUUC(f,S^T,S^C)$ be the empirical area under uplift curve of the model $f$ on the sets $S^T$ and $S^C$; and $AUUC(f)=\mathbb{E}_{S^T,S^C} \left[AUUC(f,S^T,S^C)\right]$ be its expectation. Then $AUUC(f)$ is related to ranking loss (Eq. \ref{eq:auc}) as:
\begin{equation}
\begin{aligned}
\label{eq:auuc_risks} 
AUUC(f) = \gamma-\lambda^T\mathbb{E}_{S^T}[\hat{\risk}(f,S^T)] - \lambda^C\mathbb{E}_{\tilde{S}^C}[\hat{\risk}(f,\tilde{S}^C)],
\end{aligned}
\end{equation}
where,
\begin{equation}
	\hat{\risk}(f,S^g) \triangleq \frac{1}{n^g_+n^g_-}\sum_{(\obs_i,+1)\in S^g}\sum_{(\obs_j,0)\in S^g}\indicatrice_{f(\obs_i)<f(\obs_j)}
	\label{eq:auc}
\end{equation}
is the empirical bipartite ranking risk, $g\in\{T,C\}$, $n^g_+,n^g_-$ are the amounts of positives and negatives respectively in the set $S^g$ (i.e. $n^g=n^g_-+n^g_+$), and $\gamma = \bar{y}^T - \frac{(\bar{y}^T)^2}{2} - \frac{(\bar{y}^C)^2}{2}$. 
\end{pro}

\begin{proof}
Let us remind the Definition 1 from the main paper for the empirical estimate of the AUUC measure of function $f$ over datasets $S^T$ and $S^C$~:
\begin{equation*}
AUUC(f, S^T,S^C) = \int_0^1 V(f,x)dx
\end{equation*}

\cite[Eq. 13]{surry2011quality} allows us to express $V(f,x)$ as a difference of \textit{cumulative outcome rates}
(for the formal definition please refer to \cite{surry2011quality}) of collections $S^T$ and $S^C$ respectively, induced by model $f$:

$$V(f,x) = F^{S^T}_f(x) - F^{S^C}_f(x)$$
Hence,
\begin{align}
AUUC(f, S^T,S^C) & = \int_0^1 V(f,x)dx = \int_0^1 \Bigl(F^{S^T}_f(x) - F^{S^C}_f(x)\Bigr)dx \nonumber\\
&= \int_0^1 F^{S^T}_f(x)dx - \int_0^1 F^{S^C}_f(x)dx \label{AUUC1}
\end{align}


By the mean while, we have from \cite[Eq. 9]{surry2011quality} a connection between $F_f^\mathcal{D}(x)$ and \textit{Gini coefficient} $G(f,\mathcal{D})$ over the dataset $\mathcal{D}$ which is:

\begin{equation}
    \label{Gini1}
G(f,\mathcal{D}) = \frac{2 \int_{0}^{1} F_f^\mathcal{D}(x) dx-\bar{y}^\mathcal{D}}{\bar{y}^\mathcal{D}(1-\bar{y}^\mathcal{D})}
\end{equation}
where $\bar{y}^\mathcal{D}$ is average outcome rate on $\mathcal{D}$. The Gini coefficient $G$ is also related to the area under ROC curve as follows \cite{tuffery2011data}:

\begin{equation}
    \label{Gini2}
G(f,\mathcal{D}) = 2 AUC(f,\mathcal{D}) - 1
\end{equation}

From \eqref{Gini1} and \eqref{Gini2}, it then comes~:
\begin{equation}
\label{Integral}
    \int_{0}^{1} F_f^\mathcal{D}(x) dx=\bar{y}^\mathcal{D}(1 - \bar{y}^\mathcal{D}) \cdot AUC(f, \mathcal{D}) + \frac{(\bar{y}^\mathcal{D})^2}{2}
\end{equation}

From \eqref{AUUC1} and \eqref{Integral} it comes~: 
\begin{flalign*}
AUUC(f, S^T,S^C) = \bar{y}^T(1 - \bar{y}^T) \cdot AUC(f, S^T) - \bar{y}^C(1 - \bar{y}^C) \cdot AUC(f, S^C) + \frac{(\bar{y}^T)^2}{2} - \frac{(\bar{y}^C)^2}{2}
\end{flalign*}
Now by reverting labels in the dataset $S^C$; i.e. $AUC(f, S^C)=(1 - AUC(f, \tilde{S}^C))$ we get
\begin{flalign*}
&AUUC(f, S^T,S^C) = \bar{y}^T(1 - \bar{y}^T) AUC(f, S^T) + \bar{y}^C(1 - \bar{y}^C)  \Bigl(1 - AUC(f, \tilde{S}^C)\Bigr) + \frac{(\bar{y}^T)^2}{2} - \frac{(\bar{y}^C)^2}{2} &&\\
&= \bar{y}^T(1 - \bar{y}^T) \cdot AUC(f, S^T) + \bar{y}^C(1 - \bar{y}^C) \cdot AUC(f, \tilde{S}^C) + \frac{(\bar{y}^T)^2}{2} + \frac{(\bar{y}^C)^2}{2} - \bar{y}^C
\end{flalign*}
Using the connection between AUC and the empirical ranking loss; $AUC(f,\mathcal{D}) \ = \ 1 - \hat{R}(f,\mathcal{D})$, we have~:
\begin{flalign*}
AUUC(f, S^T,S^C) &= \bar{y}^T(1 - \bar{y}^T) \cdot \Bigl(1 - \hat{R}(f, S^T)\Bigr) + \bar{y}^C(1 - \bar{y}^C) \cdot \Bigl(1 - \hat{R}(f, \tilde{S}^C)\Bigr) &&\\
&+ \frac{(\bar{y}^T)^2}{2} + \frac{(\bar{y}^C)^2}{2} - \bar{y}^C = \gamma - \Bigl(\lambda^T \hat{R}(f, S^T) + \lambda^C \hat{R}(f, \tilde{S}^C)\Bigr)
\end{flalign*}

where, for sake of notation, we use group $T$ and group $C$ instead of datasets ${S^T}$ and ${S^C}$ in the upper indices of $\bar{y}$; and $\lambda^T = \bar{y}^T(1 - \bar{y}^T), \lambda^C = \bar{y}^C(1 - \bar{y}^C), \gamma = \bar{y}^T - \frac{(\bar{y}^T)^2}{2} - \frac{(\bar{y}^C)^2}{2}.$

By taking the expectations in  both sides of equation we finally get~:
\[
AUUC(f)=\mathbb{E}_{S^T,S^C} \left[AUUC(f,S^T,S^C)\right]=\gamma-\Bigl(\lambda^T\mathbb{E}_{S^T}[\hat{\risk}(f,S^T)]+\lambda^C\mathbb{E}_{\tilde{S}^C}[\hat{\risk}(f,\tilde{S}^C)]\Bigr)
\]
\end{proof}

\begin{theo}
\label{thrm:AUUC}
        Let $S=(\obs_i^{y_i})_{i=1}^m\in (\Input\times \Output)^m$ be a dataset of $m$ examples drawn i.i.d. according to a probability distribution $\Dist$ over $\Input\times \Output$,  and decomposable according to treatment $S^T$ and reverted label control $\tilde{S}^C$ subsets. Let  $\mathcal{T}(S^T)$ and $\mathcal{T}(\tilde{S}^C)$ be the corresponding transformed sets. Then for any $1>\delta>0$ and $0/1$ loss $\ell:\{-1,+1\}\times\mathbb{R}\rightarrow [0,1]$, with probability at least~$(1-\delta)$ the following lower bound holds for all $f\in \mathcal \Fset_{r}$~:
        \begin{eqnarray}
        AUUC(f)\geq \gamma - \Bigl(\lambda^T \hat{R}_\ell(f,S^T)+\lambda^C \hat{R}_\ell(f,\tilde{S}^C)\Bigr)\nonumber 
        - C_\delta(\Fset_r,S^T,\tilde{S}^C) - \frac{25}{48}\biggl(\frac{\lambda^T}{n_+^T}+\frac{\lambda^C}{n_\_^C}\biggr)\log\frac{2}{\delta},
        \end{eqnarray}
{\footnotesize $ C_\delta(\Fset_r,S^T,\tilde{S}^C)\!=\!(\lambda^T\rademacher_{S^T}(\Fset_r)+\lambda^C\rademacher_{\tilde{S}^C}(\Fset_r))+ \!\left(\frac{\!\frac{5}{2}\sqrt{\rademacher_{S^T}(\Fset_r)}+\frac{5}{4}\sqrt{2r}}{\sqrt{n_+^T}}\lambda^T\!\!\!+\!\!\frac{\frac{5}{2}\sqrt{\rademacher_{\tilde{S}^C}(\Fset_r)}+\frac{5}{4}\sqrt{2r}}{\sqrt{n_\_^C}}\lambda^C\right)\!\!\sqrt{\log\frac{2}{\delta}}$} is defined with respect to local Rademacher complexities of the class of functions $\mathcal{F}_r$ estimated over the treatment and the control sets.
\end{theo}
        

\begin{proof}
From Proposition \ref{prop:one}:
\begin{equation}
    AUUC(f) = \gamma-\Bigl(\lambda^T\mathbb{E}_{S^T}[\hat{\risk}(f,S^T)]+\lambda^C\mathbb{E}_{\tilde{S}^C}[\hat{\risk}(f,\tilde{S}^C)]\Bigr)
\end{equation}

From \cite{ralaivola2015entropy}, we have the following upper bounds for each of the ranking losses hold with probability $1-\delta/2$~:
\begin{equation}
\forall \Fset_{r}, \mathbb{E}_{S^T}[\hat R(f,S^T)]\\
\leq  \mathop{\inf}_{a^T>0} \left((1+a^T)\rademacher_{S^T}(\Fset_r)+\frac{5}{4}\sqrt{\frac{2r\log\frac{2}{\delta}}{n_+^T}}+\frac{25}{16}\left(\frac{1}{3}+\frac{1}{a^T}\right)\frac{\log\frac{2}{\delta}}{n_+^T}\right)
\end{equation}
\begin{equation}
\forall \Fset_{r}, \mathbb{E}_{S^C}[\hat R(f,\tilde{S}^C)]\\
\leq  \mathop{\inf}_{a^C>0} \left((1+a^C)\rademacher_{\tilde{S}^C}(\Fset_r)+\frac{5}{4}\sqrt{\frac{2r\log\frac{2}{\delta}}{n_\_^C}}+\frac{25}{16}\left(\frac{1}{3}+\frac{1}{a^C}\right)\frac{\log\frac{2}{\delta}}{n_\_^C}\right)
\end{equation}
The infinimums of the upper-bounds are reached for respectively 
\begin{align*}
a^T&=\frac{5}{4}\sqrt{\frac{\log{\frac{2}{\delta}}}{n_+^T \rademacher_{S^T}(\Fset_r)}}\\
a^C&=\frac{5}{4}\sqrt{\frac{\log{\frac{2}{\delta}}}{n_\_^C \rademacher_{\tilde{S}^C}(\Fset_r)}}
\end{align*}
By plugging back these values into the upper-bounds the result follows from the union bound.

\end{proof}

\begin{pro}
\label{prop:two}
    Let $S^g$ be a sample of size $n^g$ with $n^g_+$ samples with positive labels and such that $\forall \obs\in~S^g; \|\obs\|\leq~R
$. Let $\Fset_{r}=\{\mathbf{x} \mapsto \mathbf{w}^\top \mathbf{x}:\|\mathbf{w}\| \leq \Lambda ; f\in \Fset:\variance f\leq r\}$, be the class of linear functions with bounded variance and bounded norm over the weights. Then for any $1>\delta>0$, the empirical local fractional Rademacher complexity of $\Fset_{r}$ over the set of pairs $\mathcal{T}(S^g)$ of size $n_+^g n_-^g$ can be bounded with probability at least $1-\frac{\delta}{2}$ by~: 
    \begin{equation}
    \label{eq:R_LF_bound} 
    \rademacher_{S^g}(\Fset_{r}) \leq \sqrt{\frac{R^{2} \Lambda^{2}}{n^g_+}} + \sqrt{\frac{\log \frac{2}{\delta}}{2n^g_+}}
    \end{equation}

\end{pro}

\begin{proof}
\begin{flalign*}
\rademacher_{S^g}(\Fset_r) &\stackrel{\text{Def. 3}}{=} \frac{1}{n^g}\expectation_{\sigma}\left[\sum_{j\in[J]}\omega_j\expectation_{X_{\Cset_j}}\left[\sup_{f\in\Fset_r}\sum_{i\in\Cset_j}\sigma_if(X_i)\right]\right] &&\\ &= [\substack{\text{change order of }\expectation s}] =  \frac{1}{n^g}\sum_{j\in[J]}\expectation_{X_{\Cset_j}}|\Cset_j|\left[\underbrace{\frac{1}{|\Cset_j|}\expectation_{\sigma}\left[\sup_{f\in\Fset_r}\sum_{i\in\Cset_j}\sigma_i f(X_i)\right]}_{{\hat{\rademacher}_{\Cset_j}}(\Fset_r)}\right] &&\\
&= \frac{1}{n^g} \sum_{j\in[J]}|\Cset_j| \underbrace{\expectation_{X_{\Cset_j}}\left[ {\hat{\rademacher}_{\Cset_j}}(\Fset_r)\right]}_{{\rademacher_{\Cset_j}}(\Fset_r)} 
\stackrel{\text{\cite[Eq. 3.14]{mohri2018foundations}}}{\leq} \frac{1}{n^g} \sum_{k=1}^{n^g_\_} n^g_+ \left({\hat{\rademacher}_{\Cset_j}}(\Fset_r) + \sqrt{\frac{\text{log} \frac{2}{\delta}}{2n^g_+}}\right) &&\\
&\stackrel{\text{\cite[Th. 4.3]{mohri2018foundations}}}{\leq} \frac{1}{n^g_+ n^g_\_} \sum_{k=1}^{n^g_\_} n^g_+\left( \sqrt{\frac{R^{2} \Lambda^{2}}{n^g_+}} + \sqrt{\frac{\text{log} \frac{2}{\delta}}{2n^g_+}}\right) 
= \sqrt{\frac{R^{2} \Lambda^{2}}{n^g_+}} + \sqrt{\frac{\text{log} \frac{2}{\delta}}{2n^g_+}}.
\end{flalign*}

\end{proof}

\subsection*{Dependency graph}
In figure \ref{fig:dep_graph} we provide an example of typical dependency graph.
\begin{figure}[h]
    \centering
    \includegraphics[width=0.7\linewidth]{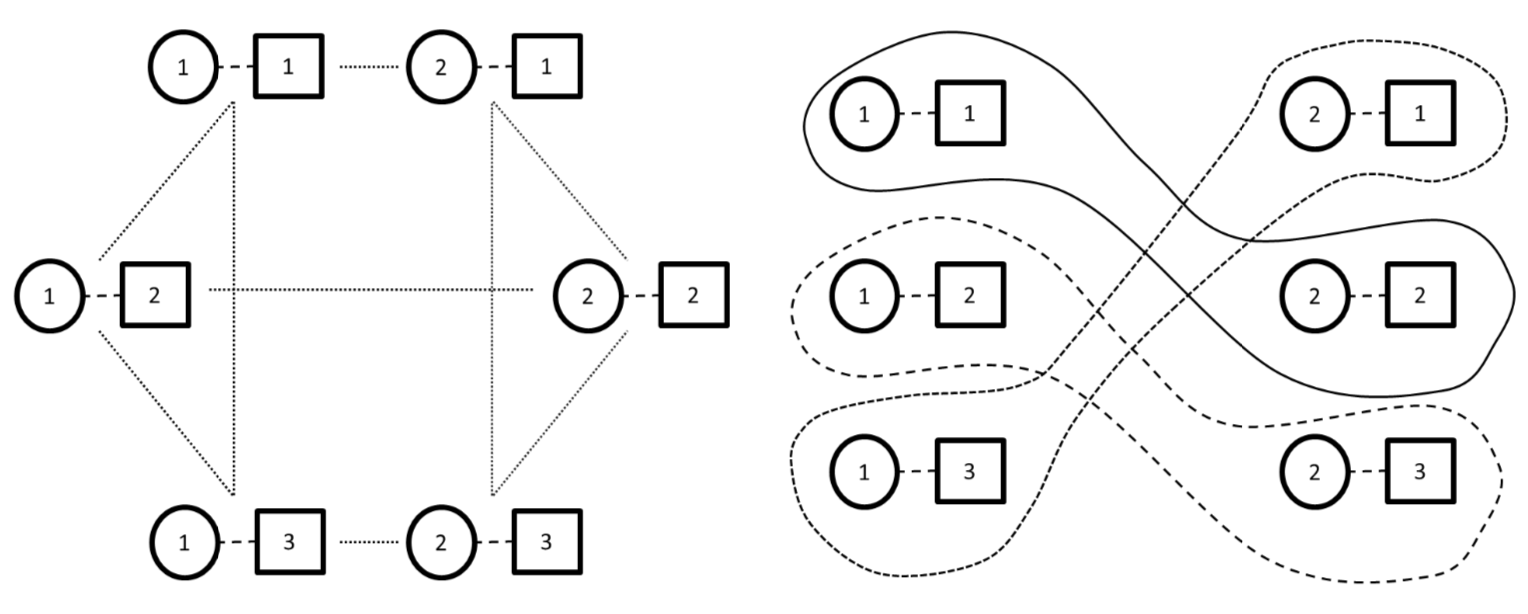}
    \captionsetup{width=.95\linewidth}
    \caption{
        \textit{(left)} Dependency graph between the pairs generated for a bipartite ranking problem with two positive examples (circles) and three negative examples (squares), \textit{(right)} description of a cover of the vertices of the dependency graph into 3 subsets containing independent pairs (figure from \cite{Springer2015}).
    }
    \label{fig:dep_graph}  
\end{figure}

\subsection*{Benchmark details}

Our benchmark consists of three open source, real-life datasets that happen to pertain to the digital marketing and social sciences applications.

\textbf{Hillstrom Email Marketing} (\href{http://www.minethatdata.com/Kevin_Hillstrom_MineThatData_E-MailAnalytics_DataMiningChallenge_2008.03.20.csv}{\color{blue}{link}}) \cite{Hillstrom2008} dataset contains results of an e-mail campaign for an Internet based retailer. We report results on the visit outcome of Women’s merchandise e-mail as treatment group versus no e-mail as control group as in previous research \cite{Rzepakowski2012b}.

\textbf{Criteo-UPLIFT2} (\href{https://s3.us-east-2.amazonaws.com/criteo-uplift-dataset/criteo-uplift-v2.csv.gz}{\color{blue}{link}}) \cite{Diemert2018} is a large scale dataset constructed from incrementality A/B tests, a particular protocol where a random part of the population is prevented from being targeted by advertising. For the speed of experiments we pick a random subsample of size 1M. We denote this variant of data as CU2-1M. 

\textbf{Jobs} \cite{lalonde1986evaluating} dataset includes 8 covariates (such as age and education, as well as previous earnings), 1 treatment (job training) and 2 outcomes (are income and employment status after training). It consists of randomized part, based on the National Supported Work program) and observational studies. We use modification of the dataset for binary classification constructed in \cite{shalit2017estimating} (\href{http://www.fredjo.com/files/jobs_DW_bin.new.10.train.npz}{\color{blue}{train link}}, \href{http://www.fredjo.com/files/jobs_DW_bin.new.10.test.npz}{\color{blue}{test link}}).

\begin{minipage}[b]{.5\linewidth}
  \centering
  \captionsetup{width=\linewidth}
  \captionof{table}{Benchmark data sets}
  \resizebox{\columnwidth}{!}{
  \begin{tabular}{lll}
    \toprule
    Data set & Hillstrom & CU2-1M\\
    \midrule
    Size & 42693 & 1000000 \\
    Features & 22 & 12 \\
    Group T ratio & 0.49905 & 0.8502 \\
    Positive class ratio & 0.12883 & 0.04897 \\
    Pos. class ratio in group T & 0.1514 & 0.04944 \\
    Pos. class ratio in group C & 0.10617 & 0.04631 \\
    Average Uplift & 0.04523 & 0.00313 \\
    \bottomrule
  \end{tabular}
  }
  \label{tab:bench}
  
\end{minipage}\hfill
\begin{minipage}[b]{.4\linewidth}
    \centering
    \includegraphics[width=\linewidth]{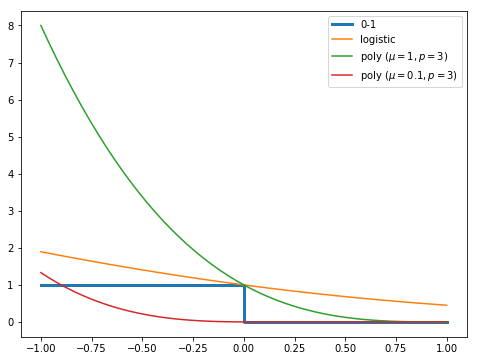}
    \captionsetup{skip=0pt, width=1.2\linewidth}
    \captionof{figure}{Surrogates for indicator function.}
    \label{fig:surr}
\end{minipage}

\subsection*{Experimental Setup details}
Technically we implemented all surrogate losses and methods (except SVM-DP\footnote{We used code of the original paper from \url{https://ftp.cs.wisc.edu/machine-learning/shavlik-group/kuusisto.ecml14.svmuplcode.zip}} and MML\footnote{We used code of the original paper from \url{https://github.com/i-yamane/uplift/blob/master/uplift/with_separate_labels/_minmax_linear.py}}) in Tensorflow framework \cite{abadi2016tensorflow}. For the optimization, Adam algorithm was used with step decay to update the learning rate.

\textbf{AUUC evaluation on Hillstrom and CU2-1M.}
For this experiments, 100 random train+validation/test splits stratified by treatment variable were used. Inside each split, for hyperparameter selection we did grid search using 5-fold cross-validation also stratified by treatment (an exception is AUUC-max, where we omitted cross-validation due to the design of the algorithm).
For all models batch size was 1000 and we ran 200 epochs of learning with early stopping by loss on validation. Grids of the hyperparameters for Hillstrom and CU2-1M datasets are provided in Table \ref{tab:params_grids_hill} and Table \ref{tab:params_grids_cu2} respectively.

\textbf{Policy Risk evaluation on Jobs.}
For the Jobs dataset, 10 train/test
splits provided in \cite{shalit2017estimating} were used. Inside each split, for hyperparameter selection we applied randomized grid search (with 50 combinations) using 1 train/validation split (except AUUC-max). Grid of the hyperparameters for Jobs dataset is provided in Table \ref{tab:params_grids_jobs}.

\begin{table}[!h]
  \caption{Hyperparameters grids for Hillstrom data}
  \label{tab:params_grids_hill}
  \centering
  \resizebox{\columnwidth}{!}{
  \begin{tabular}{llllllll}
    \toprule
    Parameter & Learning rate & $l_2$ reg. term & $\Lambda$ & $C$ & $h_{dim}$ & $\alpha$\\
    \midrule
    TM & $[1\mathrm{e}^{-1}, 5\mathrm{e}^{-1}]$  & $[0, 1\mathrm{e}^{-6}, 1\mathrm{e}^{-4}]$ & - & - & - & - \\
    CVT & $[5\mathrm{e}^{-3}, 1\mathrm{e}^{-2}]$  & $[0, 1\mathrm{e}^{-6}, 1\mathrm{e}^{-4}]$ & - & - & - & - \\
    DDR & $[1\mathrm{e}^{-1}, 5\mathrm{e}^{-1}]$  & $[0, 1\mathrm{e}^{-6}, 1\mathrm{e}^{-4}]$ & - & - & - & - \\
    SDR & $[1\mathrm{e}^{-1}, 5\mathrm{e}^{-1}]$  & $[0, 1\mathrm{e}^{-6}, 1\mathrm{e}^{-4}]$ & - & - & - & - \\
    TARNet & $[1\mathrm{e}^{-3}, 5\mathrm{e}^{-3}]$  & $[0, 1\mathrm{e}^{-6}, 1\mathrm{e}^{-4}]$ & - & - & - & - \\
    AUUC-max (all) & $[5\mathrm{e}^{-4}, 1\mathrm{e}^{-3}]$  & - & $[5\mathrm{e}^{-1}, 8\mathrm{e}^{-1}, 1\mathrm{e}^{0}]$ & - & - & - \\
    SVM-DP & - & - & - & $[1\mathrm{e}^{-1}, 1\mathrm{e}^{0}, 1\mathrm{e}^{1}]$ & - & - \\
    GANITE & - & - & - & - & $[3\mathrm{e}^{1}, 5\mathrm{e}^{1}, 1\mathrm{e}^{2}]$ & $[1\mathrm{e}^{0}, 1\mathrm{e}^{2}, 1\mathrm{e}^{3}]$ \\
    
    \bottomrule
  \end{tabular}
  }
\end{table}

\begin{table}[!h]
  \caption{Hyperparameters grids for CU2-1M data}
  \label{tab:params_grids_cu2}
  \centering
  \resizebox{\columnwidth}{!}{
  \begin{tabular}{llllll}
    \toprule
    Parameter & batch size & learning rate & $l_2$ reg. term & $\Lambda$ & $C$\\
    \midrule
    TM & $[1\mathrm{e}^{-1}, 5\mathrm{e}^{-1}]$  & $[0, 1\mathrm{e}^{-6}, 1\mathrm{e}^{-4}]$ & - & - \\
    SDR & $[1\mathrm{e}^{-1}, 5\mathrm{e}^{-1}]$  & $[0, 1\mathrm{e}^{-6}, 1\mathrm{e}^{-4}]$ & - & - \\
    AUUC-max ($s_{poly}$) & $[5\mathrm{e}^{-3}, 1\mathrm{e}^{-2}]$ & - & $[5\mathrm{e}^{-1}, 1\mathrm{e}^{0}, 5\mathrm{e}^{0}]$ & - \\
    SVM-DP & - & - & - & $[1\mathrm{e}^{-1}, 1\mathrm{e}^{0}, 1\mathrm{e}^{1}]$ \\
    
    \bottomrule
  \end{tabular}
  }
\end{table}

\begin{table}[!h]
  \caption{Hyperparameters grids for Jobs data}
  \label{tab:params_grids_jobs}
  \centering
  \resizebox{\columnwidth}{!}{
  \begin{tabular}{llllllll}
    \toprule
    Parameter & learning rate & $l_2$ reg. term & $\Lambda$ & batch size & nb. of epochs & $h_{dim}$ & $\alpha$\\
    \midrule
    TM & $[1\mathrm{e}^{-4}, 1\mathrm{e}^{-3}, 1\mathrm{e}^{-2}]$  & $[0, 1\mathrm{e}^{-6}, 1\mathrm{e}^{-4}, 1\mathrm{e}^{-2}]$ & - & $[1\mathrm{e}^{2}, 5\mathrm{e}^{2}, 1\mathrm{e}^{3}]$ & $[5\mathrm{e}^{1}, 1\mathrm{e}^{2}, 5\mathrm{e}^{2}]$ & - & - \\
    SDR & $[1\mathrm{e}^{-4}, 1\mathrm{e}^{-3}, 1\mathrm{e}^{-2}]$  & $[0, 1\mathrm{e}^{-6}, 1\mathrm{e}^{-4}, 1\mathrm{e}^{-2}]$ & - & $[1\mathrm{e}^{2}, 5\mathrm{e}^{2}, 1\mathrm{e}^{3}]$ & $[5\mathrm{e}^{1}, 1\mathrm{e}^{2}, 5\mathrm{e}^{2}]$ & - & - \\
    TARNet & $[1\mathrm{e}^{-4}, 1\mathrm{e}^{-3}, 1\mathrm{e}^{-2}]$  & $[0, 1\mathrm{e}^{-6}, 1\mathrm{e}^{-4}, 1\mathrm{e}^{-2}]$ & - & $[1\mathrm{e}^{2}, 5\mathrm{e}^{2}, 1\mathrm{e}^{3}]$ & $[5\mathrm{e}^{1}, 1\mathrm{e}^{2}, 5\mathrm{e}^{2}]$ & - & - \\
    AUUC-max (all) & $[1\mathrm{e}^{-4}, 1\mathrm{e}^{-3}, 1\mathrm{e}^{-2}]$  & - & $[5\mathrm{e}^{-3}, 1\mathrm{e}^{-2}, 5\mathrm{e}^{-2}, 1\mathrm{e}^{-1}, 5\mathrm{e}^{-1}]$ & $[1\mathrm{e}^{2}, 5\mathrm{e}^{2}, 1\mathrm{e}^{3}]$ & $[5\mathrm{e}^{1}, 1\mathrm{e}^{2}, 5\mathrm{e}^{2}]$ & - & - \\
    GANITE & $[1\mathrm{e}^{-5}, 1\mathrm{e}^{-4}, 1\mathrm{e}^{-3}]$  & - & - & $[1\mathrm{e}^{2}, 5\mathrm{e}^{2}, 1\mathrm{e}^{3}]$ & $[5\mathrm{e}^{1}, 1\mathrm{e}^{2}, 5\mathrm{e}^{2}]$ & $[3\mathrm{e}^{1}, 5\mathrm{e}^{1}, 1\mathrm{e}^{2}]$ & $[1\mathrm{e}^{0}, 1\mathrm{e}^{2}, 1\mathrm{e}^{3}]$ \\
    
    \bottomrule
  \end{tabular}
  }
\end{table}

\textbf{Evaluation of the generalization bound.}
To assess the tightness of our bound, we depict the distribution of the differences between the true AUUC (= $\E[AUUC]$) and the lower bound computed on the Hillstrom dataset. For that purpose, we learn an AUUC-max model and record the train and test AUUCs. 
$\E[AUUC]$ is estimated from the upper bound of an Empirical Bernstein inequality \cite{maurer2009empirical} on the test sets obtained from 4,000 random train/test splits, giving a precision greater or equal than $.001$ with probability $>.99$. 
The distribution of the generalization error modeled by the bound is then simply the difference between train and test AUUCs.

\textbf{Choice of surrogate.}
For the polynomial surrogate $s_{poly}$ for AUUC-max we used hyperparameters $(\mu=0.1, p=3)$. All mentioned kinds of surrogates are provided in the Figure \ref{fig:surr}.

\textbf{Hardware information.}
All experiments were run on a Linux machine with 32 CPUs (Intel(R) Xeon(R) Gold 6134 CPU @ 3.20GHz), with 2 threads per core, and 120Gb of RAM, with parallelising across 20 CPUs.

\subsection*{Additional experiments}

\textbf{AUUC on CU2-1M.}
For evaluation on the CU2-1M collection we select best performing methods on Hillstrom that can be trained reasonably fast. Results in Table \ref{tab:auuc_criteo} show very little variability and we find no one method to perform significantly better than another. We conjecture that properly tuning hyperparameters in validation blurs differences in learning quality on this large collection.


\textbf{Policy Risk on Hillstrom.}
AUUC-max has lower risk (better) when treatment would target the first 0 to 30\% most incremental individuals, as ranked by model predictions. When targeting more than 40\% of the population with the treatment SDR would have lower risk, followed by AUUC-max then TM. 


\begin{table}[ht]
\centering
\begin{minipage}{.677\linewidth}
\captionsetup{width=.95\linewidth}
\caption{\textbf{Test Policy Risk on \textit{Hillstrom}} \\ \underline{Underline} indicates lowest risk at this threshold}
\resizebox{0.975\columnwidth}{!}{%

\begin{tabular}{lrrrrrrrrr}
Treatment Ratio &     0.1 &     0.2 &     0.3 &     0.4 &   0.5  &  0.6 &     0.7 &     0.8 &     0.9 \\
\midrule
TM         &  0.8841 &  0.8768 &  0.8689 &  0.8610 & 0.8560  & 0.8541 &  0.8524 &  0.8514 &  0.8504 \\
SDR        &  0.8839 &  0.8762 &  0.8684 &  \underline{0.8603} & \underline{0.8551} & \underline{0.8534} &  \underline{0.8516} &  \underline{0.8507} &  \underline{0.8498} \\
AUUC-max &  \underline{0.8832} &  \underline{0.8759} &  \underline{0.8684} &  0.8607 & 0.8559 & 0.8544 &  0.8522 &  0.8512 &  0.8501 \\
\label{tab:policy_risk_hill}
\end{tabular}
}
\end{minipage}%
\begin{minipage}{.033\linewidth}\end{minipage}%
\begin{minipage}{.3\linewidth}
\captionsetup{width=.9\linewidth}
\caption{\textbf{Test AUUC on \textit{CU2-1M}}}
\label{tab:auuc_criteo}
\resizebox{0.95\columnwidth}{!}{
\begin{tabular}{lr}
Model & Test AUUC $\pm2\sigma$ \\
\midrule
TM & .00280 $\pm$ .00151\\
SVM-DP & .00280 $\pm$ .00152\\
SDR &  .00280 $\pm$ .00152\\
AUUC-max & .00279 $\pm$ .00149\\
\end{tabular}
}
\end{minipage}
\end{table}

\end{document}